\documentclass{article}
\usepackage{graphicx}

\usepackage{nac_preprint}
\usepackage[utf8]{inputenc} 
\usepackage[T1]{fontenc}    
\usepackage{hyperref}       
\usepackage{url}            
\usepackage{booktabs}       
\usepackage{amsfonts}       
\usepackage{nicefrac}       
\usepackage{microtype}      

\usepackage{booktabs} 
\usepackage{subcaption}
\usepackage{amsmath}
\usepackage{amssymb}
\usepackage{cleveref}
\usepackage{mathtools}
\usepackage[toc,page]{appendix}
\usepackage{enumitem}
\usepackage{graphics}
\usepackage{graphicx}
\usepackage{xcolor, colortbl}
\usepackage[export]{adjustbox}
\usepackage{algorithm}
\usepackage{algorithmicx}
\usepackage{multicol}
\usepackage{algpseudocode}
\usepackage[inkscapelatex=false]{svg} 

\usepackage{amsthm}
\newtheorem{lemma}{Lemma}

\definecolor{LightCyan}{rgb}{0.88,1,1}
\definecolor{LightOrange}{rgb}{1,0.76,0.86}

\algnewcommand{\LineComment}[1]{\State \(//\) #1}
\algnewcommand{\RLineComment}[1]{\State \(\triangleright\) #1}
\newlength{\algrhswidth}
\usepackage{bm}
\usepackage{multirow}
 
\usepackage{wrapfig}

\usepackage{etoolbox}
\usepackage{tikz}
\usetikzlibrary{tikzmark}
\usetikzlibrary{calc}

\errorcontextlines\maxdimen

\newcommand{\ALGtikzmarkcolor}{black}
\newcommand{\ALGtikzmarkextraindent}{4pt}
\newcommand{\ALGtikzmarkverticaloffsetstart}{-.5ex}
\newcommand{\ALGtikzmarkverticaloffsetend}{-.5ex}
\makeatletter
\newcounter{ALG@tikzmark@tempcnta} 

\newcommand\ALG@tikzmark@start{%
    \global\let\ALG@tikzmark@last\ALG@tikzmark@starttext%
    \expandafter\edef\csname ALG@tikzmark@\theALG@nested\endcsname{\theALG@tikzmark@tempcnta}%
    \tikzmark{ALG@tikzmark@start@\csname ALG@tikzmark@\theALG@nested\endcsname}%
    \addtocounter{ALG@tikzmark@tempcnta}{1}%
}

\def\ALG@tikzmark@starttext{start}
\newcommand\ALG@tikzmark@end{%
    \ifx\ALG@tikzmark@last\ALG@tikzmark@starttext
    \else
        \tikzmark{ALG@tikzmark@end@\csname ALG@tikzmark@\theALG@nested\endcsname}%
        \tikz[overlay,remember picture] \draw[\ALGtikzmarkcolor] let \p{S}=($(pic cs:ALG@tikzmark@start@\csname ALG@tikzmark@\theALG@nested\endcsname)+(\ALGtikzmarkextraindent,\ALGtikzmarkverticaloffsetstart)$), \p{E}=($(pic cs:ALG@tikzmark@end@\csname ALG@tikzmark@\theALG@nested\endcsname)+(\ALGtikzmarkextraindent,\ALGtikzmarkverticaloffsetend)$) in (\x{S},\y{S})--(\x{S},\y{E});%
    \fi
    \gdef\ALG@tikzmark@last{end}%
}

\apptocmd{\ALG@beginblock}{\ALG@tikzmark@start}{}{\errmessage{failed to patch}}
\pretocmd{\ALG@endblock}{\ALG@tikzmark@end}{}{\errmessage{failed to patch}}
\makeatother

\title{Optimizing Neurorobot Policy under Limited Demonstration Data through Preference Regret}

\author{%
Viet Dung Nguyen\\
 Rochester Institute of Technology \\ 
\texttt{vn1747@rit.edu}
\And 
Yuhang Song \\
Advanced Micro Devices, Inc. \\
\texttt{sgyson10@liverpool.ac.uk}
\And 
Anh Nguyen \\
University of Liverpool \\
\texttt{anh.nguyen@liverpool.ac.uk}
\And 
Jamison Heard \\
Rochester Institute of Technology \\
\texttt{jrheee@rit.edu}
\And 
Reynold Bailey \\
Rochester Institute of Technology \\
\texttt{rjbvcs@rit.edu}
\AND 
Alexander Ororbia \\
Rochester Institute of Technology \\
\texttt{ago@cs.rit.edu}
}

\begin{document}

\setlength{\abovedisplayskip}{0.065cm}
\setlength{\belowdisplayskip}{0pt}

\maketitle

\begin{abstract}

Robot reinforcement learning from demonstrations (RLfD) assumes that expert data is abundant; this is usually unrealistic in the real world given data scarcity as well as high collection cost. Furthermore, imitation learning algorithms assume that the data is independently and identically distributed, which ultimately results in poorer performance as gradual errors emerge and compound within test-time trajectories. We address these issues by introducing the ``master your own expertise'' (MYOE) framework, a self-imitation framework that enables robotic agents to learn complex behaviors from limited demonstration data samples. 
Inspired by human perception and action, we propose and design what we call the \emph{queryable mixture-of-preferences state space model} (QMoP-SSM), which estimates the desired goal at every time step. These desired goals are used in computing the ``preference regret'', which is used to optimize the robot control policy. Our experiments demonstrate the robustness, adaptability, and out-of-sample performance of our agent compared to other state-of-the-art RLfD schemes. \\ 
The GitHub repository that supports this work can be found at: \hyperlink{https://github.com/NACLab/neurorobot-preference-regret-learning}{https://github.com/NACLab/neurorobot-preference-regret-learning}. 

\keywords{Active inference \and Free energy principle \and Neurorobotics \and Cognitive control \and Generative world models \and Reinforcement learning}
\end{abstract}

\section{Introduction}
\label{sec:intro}

Typically, popular RLfD algorithms treat expert data as independent and identically distributed (i.i.d.), training in a hybrid manner between supervised learning~\cite{Pomerleau1988ALVINN, Ross2010ReductionImitationLearning, Ross2011DAgger} and reinforcement learning (RL)~\cite{Fujimoto2021TD3BC, Mandlekar2020HierarchicalImitationLearning} (and adversarial imitation learning~\cite{Ho2016GAIL, Giammarino2024LAIfO}). However, imitation learning (IL) and behavior cloning (BC) suffer from a problem known as cascading errors~\cite{Ross2010ReductionImitationLearning, Bagnell2015InvitationImitationLearning, Codevilla2019BehaviorCloningLimitDrivingCar, Chang2021CovariateShiftIL, Seo2024CovariateShiftBC, Nguyen2025SRAIF}, where small errors accumulate over time in actual trajectories. This causes the robot to deviate from demonstrated examples, leading to policy collapse during online execution. Furthermore, less attention has been given to online RL, in the context of neurorobotics \cite{chiel1997brain, Nguyen2025SRAIF}, where 
the robot learns through real-time interaction, given only a limited number of demonstrations.


In this work, we overcome the problem of cascading errors and optimize the robot policy in the constrained expert data RLfD setting via the adaptation of a generative model that dynamically produces ``preferred'' state trajectories; this results in the learning of a policy through ``preference regret'' minimization. Specifically, our contributions are as follows.   
\textbf{1)} We leverage model-based reinforcement learning and active inference (AIF) process theory to model perception and action, enabling online learning under limited expert demonstrations. 
\textbf{2)} We introduce the queryable mixture-of-preferences state space model (QMoP-SSM), a goal / preference predictor that estimates future desired trajectories to guide policy learning. 
\textbf{3)} We propose a novel policy / behavior learning method that makes use of this future desired trajectory by computing what we call the ``preference regret''. 
We refer to our approach as the ``master your own expertise'' (MYOE) framework. Finally, \textbf{4)} we provide empirical results that shows that our 
framework outperforms state-of-the-art imitation-learning-augmented model-free and model-based RL baselines.

\section{Related Work}
\label{sec:review}

\begin{figure}[t]
    \centering
    \vspace{0.575mm}
    \includegraphics[width=0.6\linewidth]{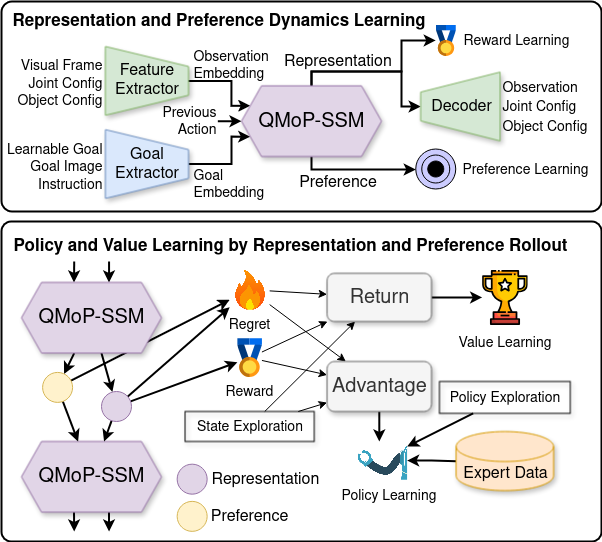}
    \caption{\textbf{Our Proposed Agent Framework.} The agent learns internal representations via encoders, the QMoP-SSM, and a decoder. Imagined future states and preferences guide both policy learning and value estimation.}
    \label{fig:overview}
    \vspace{-6mm}
\end{figure}

\noindent 
\textbf{Perception Modeling.} This line of work focuses on biologically-inspired design patterns that aim to provide an agent with an understanding of its niche, or ``lived'' world~\cite{Conant1970EverythingMustBeModel, Friston2013LifeAsWeKnowIt, Ororbia2024MortalComputation}. This concept has been studied across different research domains such as model-based reinforcement learning~\cite{Sutton2018RLBook}, active inference~\cite{Parr2022AIFFEPBook, Ororbia2022BackpropFreeRL}, and variational inference~\cite{Kingma2014VAE}. Central to these approaches is the ``generative world model'', which encapsulates two key concepts: 
\textbf{1)} inferring internal beliefs from observations, and 
\textbf{2)} predicting future observations given those beliefs. To achieve this, world models and recurrent state space models aim to minimize the surprise with respect to model estimations as well as maximize the observation prediction accuracy~\cite{Ha2018WorldModel, Hafner2020DreamerV1}, an idea often referred to as the free energy principle~\cite{Friston2010FEP, Friston2017AIFProcessTheory}:
\begin{equation}
\begin{aligned}
    F = \text{D}_\text{KL}[ q(s_t | o_t) \parallel p (s_t | s_{t-1}, a_{t-1}) ] - \mathbb{E}_{q(s_t)} [ \ln q(o_t | s_t) ]
\end{aligned}
\end{equation}
where $o_t$, $s_t$, and $a_t$ are the observation, hidden state, and action / control taken at time step $t$, respectively. Similar to minimizing the evidence lower bound (ELBO)~\cite{Kingma2014VAE}, the first term minimizes the Kullback–Leibler (KL) divergence between the approximated posterior and the prior over the hidden state 
whereas the second term maximizes the system's prediction accuracy~\cite{Friston2017AIFProcessTheory}. Later work integrates recurrence into the planning process, e.g., using a ``state space model''~\cite{Ha2018WorldModel, Hansen2024TDMPC2}, which can 
enhance an agent's ability to form internal beliefs through a perception model. 
In AIF studies, some agents solve tasks by employing abstract representations of desired goals, a notion referred to as learning prior preferences. In this work, we use the term ``preference'' to denote the learned / predicted trajectory of prior preferences. 

\noindent 
\textbf{Imitation Learning.} Imitation learning (IL) methodology originated from ALVINN, an early self-driving system that was designed to mimic human driving behavior~\cite{Pomerleau1988ALVINN}. This led to the development of behavioral cloning (BC), which frames IL as a supervised learning problem~\cite{Pomerleau1991OriginalBehaviorCloning},  where observations are inputs and expert actions are labels used to train a ``cloned'' policy. However, BC suffers from instability due to the problem of ``cascading errors'', where small planning mistakes accumulate during real-world execution, driving the agent into trajectory distributions that have not been observed in the training data~\cite{Ross2010ReductionImitationLearning, Bagnell2015InvitationImitationLearning}. This issue, also referred to as ``covariate shift''~\cite{Zare2024ImitationLearningSurvey} in the underlying dataset distribution, has been mitigated through algorithms such as the dataset aggregation (DAgger) method~\cite{Ross2011DAgger}, inverse reinforcement learning~\cite{Zare2024ImitationLearningSurvey}, generative adversarial methods~\cite{Ho2016GAIL, Giammarino2024LAIfO}, and self-imitation learning~\cite{Ferret2021SAIL}. 
However, relatively little work has explored scenarios where robotic agents must learn from a \textbf{limited number of demonstrations}, a more practical setting with a lower cost in terms of time, effort, and computational resources. In this work, we tackle this challenge by leveraging the concept of perception modeling as well as AIF process theory to construct agents that learn from limited expert demonstrations while mitigating the impact of cascading errors in the context of online learning.

\section{Robot Learning through Preference Regret}
\label{sec:method} 

Our agent operates on partially-observable Markov decision processes (POMDPs)~\cite{Sutton2018RLBook} where the observation space includes visual sensory inputs, proprioceptive states, object states, and desired goals. Agents must learn in real-time through environmental interactions and are only provided with a small number (five) episodes of expert demonstrations. To solve robotic tasks, our approach leverages perception modeling, enabling the ability to deduce both future observations and goal trajectories; see Section~\ref{sec:qmop}. Based on the derived future representations and preferences, we train our system's behavioral policy by minimizing a term called ``preference regret''. This term is computed from the resulting goal trajectory from our proposed perception model; it is built on the advantage estimation of proximal policy optimization (PPO) frameworks~\cite{Sutton2018RLBook, Schulman2015GeneralGAELambda} and the behavioral learning element of AIF process theory~\cite{Friston2015Epistemic, Friston2017Curiosity, Millidge2021WhenceEFE}; see Section~\ref{sec:behavior}. 
The training process can be viewed in Figure~\ref{fig:overview}.

\begin{figure}[!t]
    \centering
    \vspace{0.2mm}
    \includegraphics[width=0.7\linewidth]{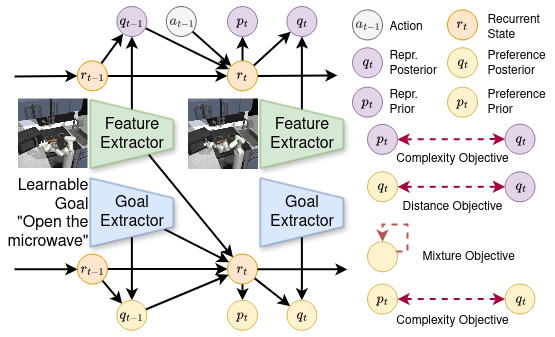}
    \caption{\textbf{QMoP-SSM Learning Architecture.} QMoP separates latent learning into two interconnected parallel processes: ``representation'' and ``preference'' learning. While the model predicts future latent given previous actions, future preference state trajectories are guided by provided and learned goals.
    }
    \label{fig:qmop}
    \vspace{-6mm}
\end{figure}

\subsection{The Queryable Mixture-of-Preferences State Space Model}
\label{sec:qmop}

Motivated by efforts in active inference and prior preference learning~\cite{Shin2022PriorPreferenceLearning, Sajid2022PriorPreferenceLearning}, we construct an agent that predicts future preference trajectories based on its own learned goal representations. Consequently, agent policy adaptation 
is influenced by both ``imagined'' future preferences as well as the expert data, which mitigates the cascading errors effect. 
Concretely, provided any goal, (e.g., language instructions, goal images / states, or learnable goal vectors), the QMoP-SSM ``queries'' the agent's goals and then predicts a roll-out of future (latent) representations and (prior) preferences. These trajectories then drive the agent's policy toward desired goal states while maximizing the utility values contained within expert data as well as optimizing final rewards; see Figures~\ref{fig:overview} and~\ref{fig:qmop} for details of the agent and its learning process.

\noindent 
\textbf{Formulation.} Our framework integrates world model building~\cite{Ha2018WorldModel, Buesing2018RLGenerativeModel, Hafner2019Planet}, (recurrent) state space models~\cite{Doerr2018PRSSM, kalman1960filter, Krishnan2015DeepKalmanFilters, Hafner2020DreamerV1}, and active inference~\cite{Friston2017AIFProcessTheory}, yielding an agent model that continuously predicts future states to guide its ability to plan actions. Our perception model, QMoP-SSM, which also predicts future ``preferences'' (or representations of optimized states), is expressed formally as follows:
\begin{equation}\label{eqn:wm}
\begin{aligned}
    \text{Representation Posterior: } &q(s^o_t | o_t, h_t)\\
    \text{Representation Prior: } &p(s^o_t | s^o_{t-1}, h_{t-1}, a_{t-1})\\
    \text{Representation Likelihood: } &p(o_t|s^o_t, h_t)\\
    \text{Representation Reward: } &p(r_t | s^o_t, h_t)\\
    \text{Preference Posterior: } &q(s^p_t | o_t, h_t)\\
    \text{Preference Prior: } &p(s^p_t | s^p_{t-1}, g^\theta_{t-1}, h_{t-1})\\
    \text{Preference Reward: } &p(r_t | s^p_t, h_t).
\end{aligned}
\end{equation}
where $o_t$ is the observation (at time $t$), $h_t$ is the recurrent hidden state, $s^o_t$ is the world (representation) state, $s^p_t$ is the preference state, and $r_t$ is the (extrinsic) reward. 
$g^\theta$ is the learnable goal, which is optimized via preference learning in our perception model. In this work, we use the term $g^\theta$ to denote the tensor that encapsulates all provided goal modalities. These modalities may include (but are not limited to) learnable tensors, natural language instructions, and goal/target images. 
Finally, in line with deep AIF~\cite{Mazzaglia2022DeepAIFSurvey} and recurrent state space modeling~\cite{Ha2018WorldModel, Hafner2020DreamerV1} methodology, we parameterize the representation and preference modules of our agent with artificial neural networks (ANNs) that are conditioned on a gated recurrent unit (GRU) recurrent network's hidden neural states~\cite{chung2014gru}. Figure~\ref{fig:qmop} illustrates the inference and learning components of our QMoP-SSM.

The agent first learns the representation model by continuously adjusting its state $s^o$ so as to minimize the surprisal between the approximated posterior $q(s^o_t | o_t, h_t)$ and the prior $p(s^o_t | s^o_{t-1}, h_{t-1}, a_{t-1})$ while furthermore maximizing its observation prediction accuracy from the world representation likelihood $p(o_t | s^o_t, h_t)$. This process resembles maximizing the ELBO~\cite{Kingma2014VAE, Hoffman2013StochasticVariationalInference} as well as minimizing what is known as the marginal free energy~\cite{Parr2019MarginalMessagePassing, Friston2016AIFandLearning, Friston2017AIFProcessTheory, Smith2022ActiveInferenceTutorial}. 
Next, our system learns to infer the preference states by predicting the successful trajectories, which are known from the expert demonstrations or collected from the agent's interaction with the environment. 
To achieve this, we minimize the distance between the preference state distribution and the actual representation state distribution. Formally, given that the perception model QMoP-SSM is parameterized by $\theta$, the free energy (learning) objective can be described as follows: 
\begin{subequations}
\begin{align}
&\underset{\theta}{\text{arg} \min}\ \mathcal{F}_t(\theta) = \mathcal{F}_o + \mathcal{F}_{o,\text{KL}} + \mathcal{F}_r + \mathcal{F}_{p,\text{KL}} + \mathcal{F}_{\text{dist}}\\
\mathcal{F}_o &= - \mathbb{E}_{q_\theta(s^o_t)}[ \ln ( p_\theta( o_t | s^o_t, h_t) ] \label{eqn:wm-acc}\\
\mathcal{F}_{o,\text{KL}} &= \text{D}_\text{KL} [ q_\theta(s^o_t | o_t, h_t) \parallel  p_\theta(s^o_t | s^o_{t-1}, a_{t-1}, h_{t-1}) ]\label{eqn:wm-compl} \\
\mathcal{F}_r &= - \mathbb{E}_{q_\theta(s^o_t)}[ \ln ( p_\theta( r_t | s^o_t, h_t) ] \label{eqn:wm-rew}\\
\mathcal{F}_{p,\text{KL}} &=  \text{D}_\text{KL} [ q(s^p_t | o_t, h_t) \parallel  p (s^p_t | s^p_{t-1}, g^\theta, h_{t-1}) ] \odot m_t \label{eqn:pref-compl} \\
\mathcal{F}_{\text{dist}} &= \mathcal{D} [ q(s^p_t | o_t, h_t) \parallel q(s^o_t | o_t, h_t) ] \odot m_t . \label{eqn:pref-dist}
\end{align}
\end{subequations}
In the above, the terms in Equations~\ref{eqn:wm-acc} and~\ref{eqn:wm-rew} seek to maximize prediction accuracy for observations and rewards. 
The terms in Equations~\ref{eqn:wm-compl} and~\ref{eqn:pref-compl} minimize the KL divergence, (i.e., the amount of surprise), between the approximate posterior and the prior distribution. Finally, the term in Equation~\ref{eqn:pref-dist} aims to minimize the distance between the posterior over the states of the preference and the representation models. 
This helps the future predicted preference state distributions align with the successful state distributions of the representation model. Choices for distance function $\mathcal{D}$ can include 
the mean (squared) Euclidean distance, negative cosine similarity, and Manhattan distance. 
Note that the (binary) mask $m_t$ represents whether the state at $t$ is part of a successful trajectory (or not) and is element-wise multiplied ($\odot$) with the corresponding preference objectives in order to constrain the preference model to only predict the next goal states. Finally, when learning how to predict rewards, the preference reward objective $-\mathbb{E}_{q_\theta(s^p_t)}[ \ln ( p_\theta( r_t | s^p_t, h_t) ]$ can be omitted since the weights may be shared among the reward models $p_\theta(r_t|s^o_t, h_t)$ and $p_\theta(r_t|s^p_t, h_t)$ (i.e., the preference state aims to estimate the representation state within successful trajectories).


\noindent 
\textbf{Mixture of Preferences.} As discussed earlier, our proposed state space model (SSM) predicts future preference trajectories that are guided by goal information. However, multiple feasible trajectories can lead to the same goal (state); this introduces noise into the preference learning process, potentially causing mode collapse \cite{Kossale2022ModeCollapseGAN}. To address this challenge, we extend the SSM to operate as a mixture of stochastic distributions for both the prior and posterior. Specifically, given a mixture of size $M \in \mathbb{N}$, we then estimate a collection of preference state distributions $\{p^i_\theta(s^p_t | s^p_{t-1}, q^\theta, h_{t-1})\}_{i=1}^M$. This formulation provides three key advantages: 
\textbf{1)} varied-goal environments -- an individual mixture component can specialize for a specific trajectory type, facilitating comprehensive coverage of all preferred behaviors; 
\textbf{2)} multiple preference trajectories expand the desired behaviors' state space, which prevents mode collapse in the estimated preference distributions; and  
\textbf{3)} a diverse mixture of preferences yields high entropy -- this promotes exploration as predicted future preferences are used as the learning signal to drive policy adaptation. 
Note that this also aligns with past RL effort that has demonstrated that entropy maximization enhances exploration~\cite{Haarnoja2018SAC,Haarnoja2018SACAPP, Barto2013IntrinsicMotivation, Oudeyer2007IntrinsicMotivation}.

Formally, given a mixture of $M$ predicted preference distributions $\{p^i_\theta(s^p_t | s^p_{t-1}, q^\theta, h_{t-1})\}_{i=1}^M$, we designate / choose one by computing a weighted linear combination over all of the distributions' parameters within the mixture:
\begin{equation}
\begin{aligned}
    \bar{p}_\theta(s^p_t) = \sum_{i=1}^{M} { \left[ \sigma( \mathbf{z} ) \odot s^{p,i}_t \right] }, \quad \mathbf{z} = W \big({q^\theta}\big)^\top + b 
\end{aligned}
\end{equation}
where $s^{p,i}_t \sim p^i_\theta(s^p_t | s^p_{t-1}, q^\theta, h_{t-1})$ and 
$q^\theta \in \mathbb{R}^{d_q}$ is the (learnable) encoded goal / query, $W \in \mathbb{R}^{M \times d_q}$ and $b \in \mathbb{R}^{M \times 1}$ produce the logit $\mathbf{z}$, and $\sigma(\mathbf{z})$ is the softmax function applied to the mapped logit component so as to produce the mixture probability. For each preference trajectory, we seek high modality coverage, avoidance of mode collapse, and exploratory signals for downstream policy learning. To achieve these objectives, we maximize the entropy of the predicted mixture of preferences in the following manner:
\begin{equation}
\begin{aligned}
    \underset{\theta}{\text{arg} \max}\ \mathcal{L}_t(\theta) = \mathbf{H} \left[ \bar{p}_\theta(s^p_t) \right] - \alpha \parallel \bar{s}^p_t \parallel_2, \quad \bar{s}^p_t \sim \bar{p}_\theta(s^p_t)
\end{aligned}
\end{equation}
where $\mathbf{H} \left[ \bar{p}_\theta(s^p_t) \right]$ is the entropy of the predicted preference distribution and $\parallel \bar{s}^p_t \parallel_2$ is the L2 regularization term used to stabilize the preference prediction. Finally, $\alpha > 0$ is the coefficient to stabilize the learning objective, e.g., $\alpha = 0.1$. This last objective enlarges preference mode coverage while maintaining the smoothness of the combined distribution.


\subsection{MYOE: Behavioral Learning through Preference Regret}
\label{sec:behavior}

In order to enable a form of effective behavior learning that leverages expert demonstrations while further still maximizing the final rewards obtained via environmental interactions, we employ a preference-guided policy optimization framework based on regret minimization principles \cite{Zinkevich2007RegretMinimization}. As a result, we integrate both the predicted representation $p(s^o)$ and preference $p(s^p)$ distributions produced by the QMoP-SSM to guide the adaptation of our system's policy. Specifically, to train the behavioral model, we utilize policy gradients~\cite{Sutton1999PolicyGrad, Ccatal2019BayesianPolicyAIF, Millidge2020VariationalPolicyGradients}, actor-critic~\cite{Konda2003ActorCritic, Sutton2018RLBook, Haarnoja2018SAC, Hafner2020DreamerV1, Hansen2024TDMPC2}, and generalized advantage estimation (GAE-$\lambda$)~\cite{Schulman2015GeneralGAELambda}, leveraging state-action-value estimation to nudge the agent toward trajectories with higher rewards. Formally, we consider the following modules:
\begin{equation}
\begin{aligned}
    \text{Policy: } &\pi_\psi(a_\tau | s^o_\tau, s^p_\tau, h_\tau)\\
    \text{Value: } V_\nu(v_\tau | s^o_\tau, h_\tau);\ &\text{Target Value: } V^{'}_\nu(v_\tau | s^o_\tau, h_\tau). \nonumber 
\end{aligned}
\end{equation}
where $\psi$, $\nu$, and $\nu'$ denote the parameters of the policy network, the value network, and the target value network, respectively. Given the future planning step $\tau$, imagination horizon $H$, and the learned goal $g^\theta$, the agent works to predict sequences of future representations of observation and preference state distributions as follows:
\begin{equation}
\{ ( p_\theta(s^o_\tau | a_{\tau - 1}, h_{t-1}), \  \bar{p}_\theta(s^p_\tau | s^p_{\tau - 1}, g^\theta, h_{t-1}) ) \}_{\tau=1}^H
\end{equation}
where $a_\tau \sim \pi_\psi(a_\tau | s^o_\tau, s^p_\tau, h_\tau)$ represents actions sampled from the current policy and $\bar{p}_\theta(s^p_\tau | s^p_{\tau - 1}, g^\theta, h_{t-1}) )$ denotes the QMoP-SSM's mixture preference distribution. This formulation allows downstream behavior learning to benefit from both environmental dynamics and the learned preferences.

\begin{figure*}[!t]
\centering
\begin{subfigure}{0.11\linewidth}
    \includegraphics[width=\linewidth]{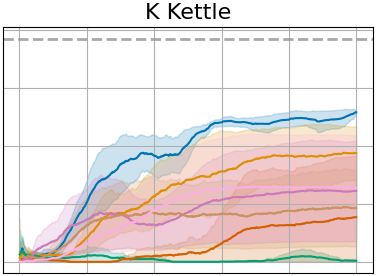}
\end{subfigure}
\begin{subfigure}{0.11\linewidth}
    \includegraphics[width=\linewidth]{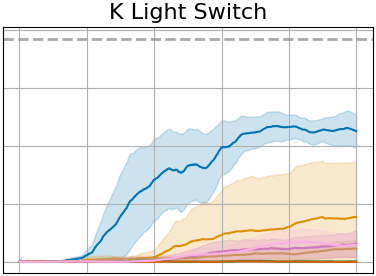}
\end{subfigure}
\begin{subfigure}{0.11\linewidth}
    \includegraphics[width=\linewidth]{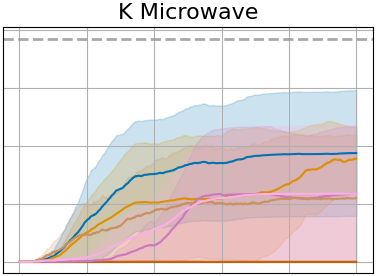}
\end{subfigure}
\begin{subfigure}{0.11\linewidth}
    \includegraphics[width=\linewidth]{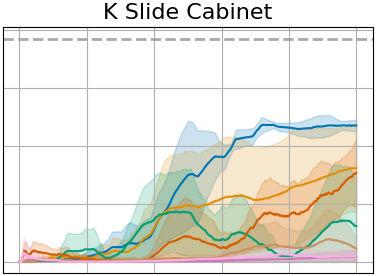}
\end{subfigure}
\begin{subfigure}{0.11\linewidth}
    \includegraphics[width=\linewidth]{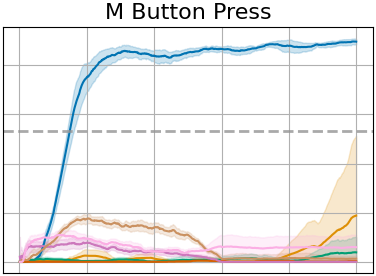}
\end{subfigure}
\begin{subfigure}{0.11\linewidth}
    \includegraphics[width=\linewidth]{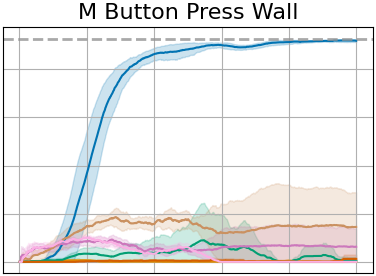}
\end{subfigure}
\begin{subfigure}{0.11\linewidth}
    \includegraphics[width=\linewidth]{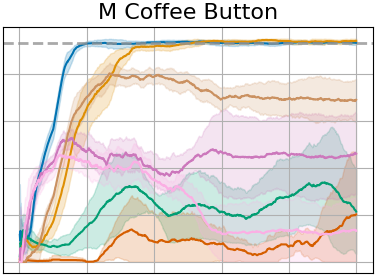}
\end{subfigure}
\begin{subfigure}{0.11\linewidth}
    \includegraphics[width=\linewidth]{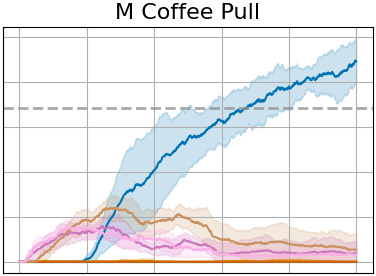}
\end{subfigure}
\begin{subfigure}{0.11\linewidth}
    \includegraphics[width=\linewidth]{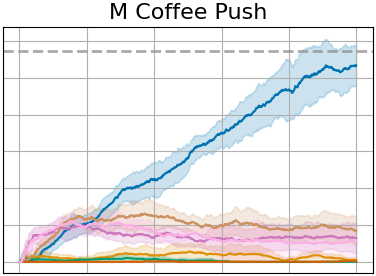}
\end{subfigure}
\begin{subfigure}{0.11\linewidth}
    \includegraphics[width=\linewidth]{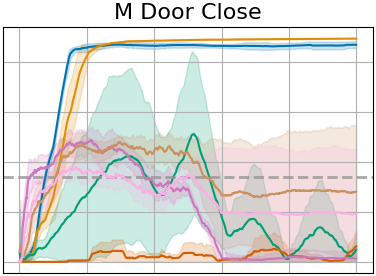}
\end{subfigure}
\begin{subfigure}{0.11\linewidth}
    \includegraphics[width=\linewidth]{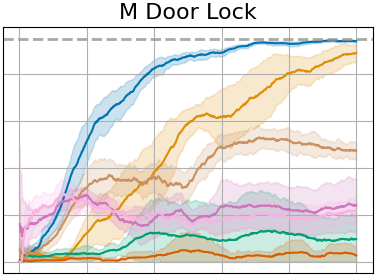}
\end{subfigure}
\begin{subfigure}{0.11\linewidth}
    \includegraphics[width=\linewidth]{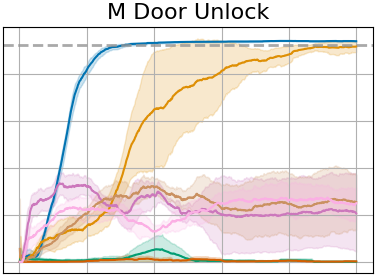}
\end{subfigure}
\begin{subfigure}{0.11\linewidth}
    \includegraphics[width=\linewidth]{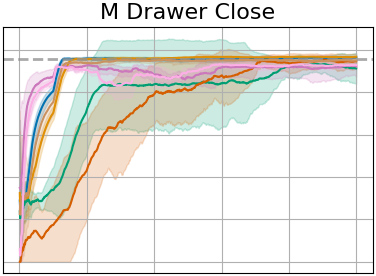}
\end{subfigure}
\begin{subfigure}{0.11\linewidth}
    \includegraphics[width=\linewidth]{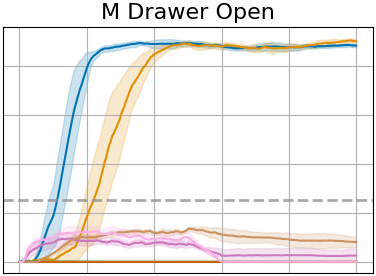}
\end{subfigure}
\begin{subfigure}{0.11\linewidth}
    \includegraphics[width=\linewidth]{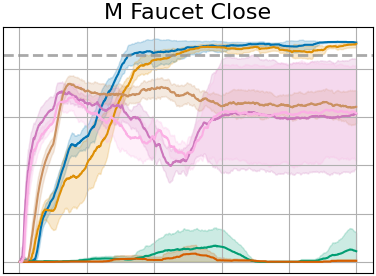}
\end{subfigure}
\begin{subfigure}{0.11\linewidth}
    \includegraphics[width=\linewidth]{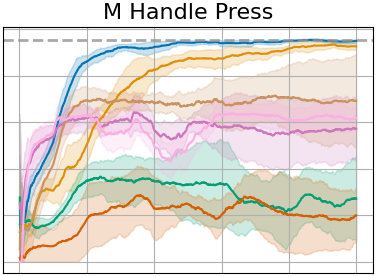}
\end{subfigure}
\begin{subfigure}{0.11\linewidth}
    \includegraphics[width=\linewidth]{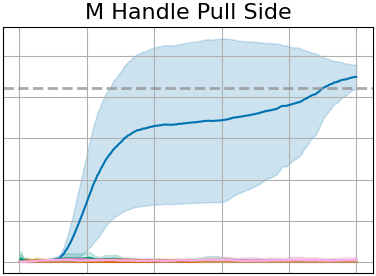}
\end{subfigure}
\begin{subfigure}{0.11\linewidth}
    \includegraphics[width=\linewidth]{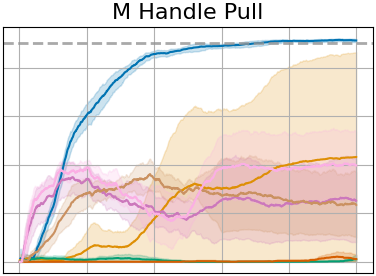}
\end{subfigure}
\begin{subfigure}{0.11\linewidth}
    \includegraphics[width=\linewidth]{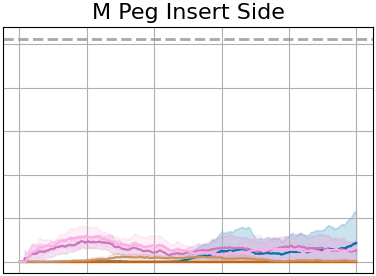}
\end{subfigure}
\begin{subfigure}{0.11\linewidth}
    \includegraphics[width=\linewidth]{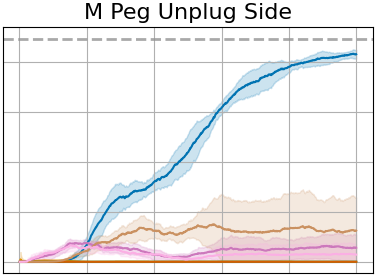}
\end{subfigure}
\begin{subfigure}{0.11\linewidth}
    \includegraphics[width=\linewidth]{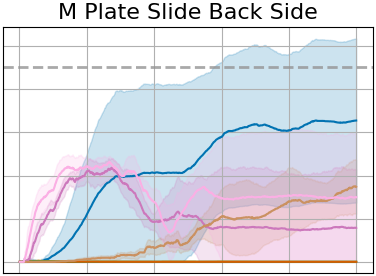}
\end{subfigure}
\begin{subfigure}{0.11\linewidth}
    \includegraphics[width=\linewidth]{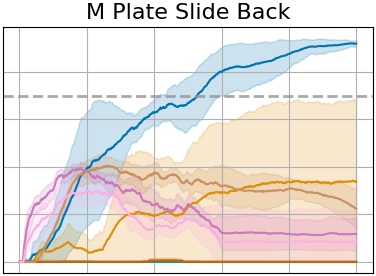}
\end{subfigure}
\begin{subfigure}{0.11\linewidth}
    \includegraphics[width=\linewidth]{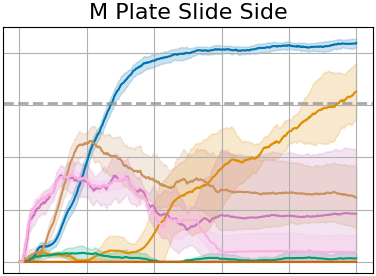}
\end{subfigure}
\begin{subfigure}{0.11\linewidth}
    \includegraphics[width=\linewidth]{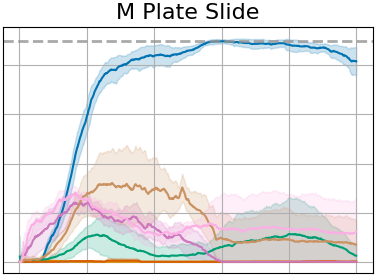}
\end{subfigure}
\begin{subfigure}{0.11\linewidth}
    \includegraphics[width=\linewidth]{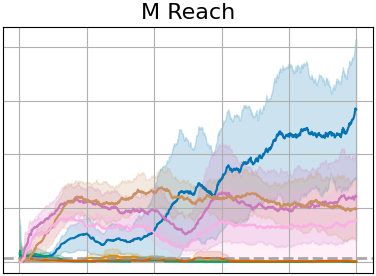}
\end{subfigure}
\begin{subfigure}{0.11\linewidth}
    \includegraphics[width=\linewidth]{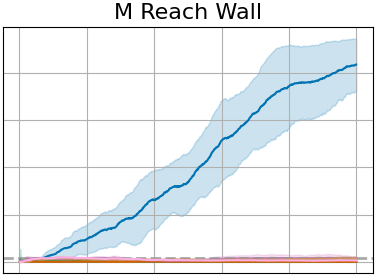}
\end{subfigure}
\begin{subfigure}{0.11\linewidth}
    \includegraphics[width=\linewidth]{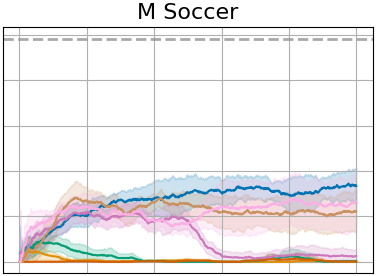}
\end{subfigure}
\begin{subfigure}{0.11\linewidth}
    \includegraphics[width=\linewidth]{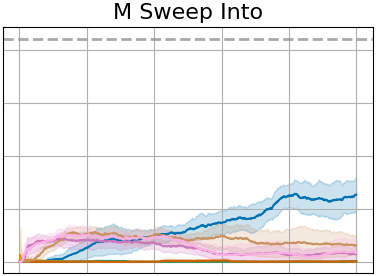}
\end{subfigure}
\begin{subfigure}{0.11\linewidth}
    \includegraphics[width=\linewidth]{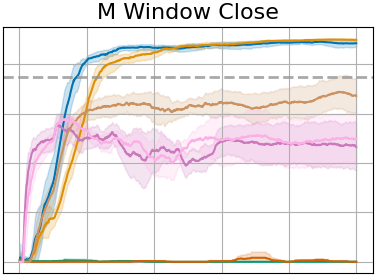}
\end{subfigure}
\begin{subfigure}{0.11\linewidth}
    \includegraphics[width=\linewidth]{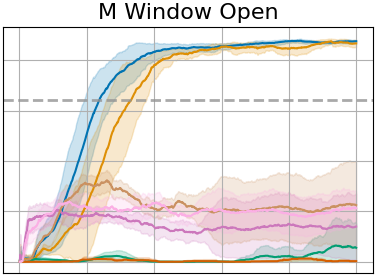}
\end{subfigure}
\begin{subfigure}{0.11\linewidth}
    \includegraphics[width=\linewidth]{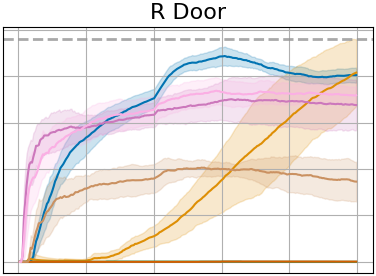}
\end{subfigure}
\begin{subfigure}{0.11\linewidth}
    \includegraphics[width=\linewidth]{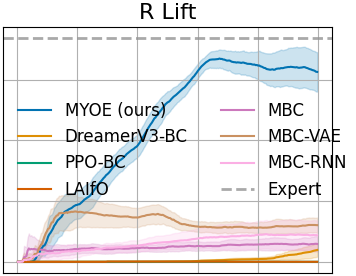}
\end{subfigure}
\caption{Cumulative reward ($y$-axis) across $1$ million interaction/training steps ($x$-axis) for different agents.}\label{fig:cumrew}
\vspace{-6mm}
\end{figure*}

\noindent 
\textbf{The Preference-Regret.} To effectively leverage the preference trajectories yielded by our QMoP-SSM, we introduce a novel intrinsic reward augmentation formulation to compute the target Q-value, inspired by the concept of ``regret minimization'' in multi-armed bandit problems~\cite{Auer2002FiniteTimeAnalysisBandit, Sutton2018RLBook, Lattimore2020BanditAlgorithmsBook}. Classical regret quantifies the difference between rewards obtained by an agent and those retrieved from optimal actions that could have been executed to interact with the environment. In our context, we extend the concept of ``regret'' by defining ``preference regret'' as the difference between actual rewards and rewards estimated using preference state trajectories. We also incorporate intrinsic motivation~\cite{Barto2013IntrinsicMotivation, Deci1985IntrinsicMotivationBook} into our framing based on the agent's state distribution entropy; this is similar to AIF process theory where the agent seeks to reduce the epistemic uncertainty within future estimated state distributions~\cite{Friston2015Epistemic, Millidge2021WhenceEFE, Friston2017Curiosity, Schwartenbeck2019Curiosity}. Thus, our agent's intrinsic reward, in line with expected free energy~\cite{Friston2015Epistemic, Friston2013LifeAsWeKnowIt, Friston2017AIFProcessTheory}, is: 
\begin{equation} 
\begin{aligned}
    \mathbf{R}_\tau = &\underbrace{r^o_\tau}_\text{reward} - \underbrace{\left( r^p_\tau - r^o_\tau \right)}_\text{preference regret} + \alpha\underbrace{\mathbf{H} \left[ p(s^o_\tau | s^o_{\tau-1}, a_{\tau-1}, h_{\tau-1}) \right]}_\text{representation state entropy} \label{eqn:intrew}
\end{aligned}
\end{equation}
where $r^o_\tau \sim p_\theta(r_\tau | s^o_\tau, h_\tau)$ and $r^p_\tau \sim p_\theta(r_\tau | s^p_\tau, h_\tau)$ represent predicted rewards conditioned on estimated future representations and preferences distributions, respectively. $a_{\tau-1} \sim \pi_\psi(a_{\tau - 1} | s^o_{\tau-1}, s^p_{\tau-1}, h_{\tau-1})$ denotes the policy's sampled actions. $\alpha \approx 0.0003$ is the scaling coefficient for stability, with its value chosen according to the soft actor-critic algorithm convention~\cite{Haarnoja2018SAC, Haarnoja2018SACAPP}. The ``preference regret'' $(r^p_\tau - r^o_\tau)$ penalizes deviations from preferred trajectories while preserving online task performance. Therefore, the behavioral model learns to solve tasks while still following the preferred trajectories guided by the learned goal(s).

Following typical temporal difference (TD) learning frameworks~\cite{Schulman2015GeneralGAELambda, Sutton2018RLBook, Hafner2025DreamerV3}, we train the value function $V_\nu$ to estimate the expected discounted return $\mathcal{G_\tau}$ (which is computed from discounted, augmented rewards $\mathbf{R_\tau}$). In order to achieve this, we first compute the TD error as follows:
\begin{equation}\label{eqn:delta}
\begin{aligned}
    \delta_\tau = \mathbf{R}_\tau + \gamma V_\nu(v_{\tau+1} | s^o_{\tau+1}, h_{\tau+1}) - V_\nu(v_{\tau} | s^o_{\tau}, h_{\tau})
\end{aligned}
\end{equation}
where $\gamma \in [0, 1]$, the discount factor, determines the relative importance of future rewards. Finally, to obtain stable low-variance policy estimates, we employ GAE-$\lambda$~\cite{Schulman2015GeneralGAELambda} to compute the advantage value $\mathcal{A_\tau}$ and the target returns $\mathcal{G_\tau}$ as below:
\begin{equation}\label{eqn:adv}
\begin{aligned}
    \mathcal{A}_\tau = \sum_{\tau=1}^H{ (\gamma \lambda)^\tau \delta_\tau };\ \mathcal{G}_\tau = \mathcal{A}_\tau + V_\nu(v_\tau | s^o_\tau, h_\tau)
\end{aligned}
\end{equation}
where $\lambda$ controls the bias-variance trade-off in advantage estimation. When $\lambda = 0$, the estimates reduce to a one-step TD method (low variance, high bias) whereas $\lambda = 1$ approaches Monte Carlo estimation (high variance, low bias). This term weights the temporal difference errors $\delta_\tau$ across the trajectory, enhancing optimization stability.

\noindent 
\textbf{Value and Policy Learning.} After computing the target $\lambda$-weighted return value $\mathcal{G}_\tau$ from Equations~\ref{eqn:intrew},~\ref{eqn:delta}, and~\ref{eqn:adv}, we train the value estimator $f_\nu$ in the following manner:
\begin{subequations}
\begin{align}
    \underset{\nu}{\text{arg} \min}\ \mathcal{L}_\tau(\nu) = \parallel V_\nu(v_\tau | s^o_\tau, h_\tau) -  \text{sg}(\mathcal{G}_\tau)\parallel_2 \label{eqn:val-main} \\
    + \alpha \parallel V_\nu(v_\tau | s^o_\tau, h_\tau) - \text{sg}\bigl( V^{'}_\nu(v_\tau|s^o_\tau, h_\tau) \bigr)\parallel_2 . \label{eqn:val-norm}
\end{align}
\end{subequations}
``sg'' denotes the stop gradient operator 
and scalar $\alpha$ induces 
training stability. Equation~\ref{eqn:val-main} minimizes the mean squared error between the predicted value and the $\lambda$-return;  Equation~\ref{eqn:val-norm} reduces the weight gap between the value and target networks (to improve training stability~\cite{Hafner2020DreamerV2,Hafner2025DreamerV3}). Target network $V'_\nu$ weights are updated (per step) using the source value network $V_\nu$ via an exponential moving average~\cite{Lillicrap2015DDPG}.

Next, policy learning is performed through advantage value maximization via policy gradients as follows: 
\begin{subequations}
\begin{align}
    \underset{\psi}{\text{arg} \max}\ \mathcal{L}_\tau(\psi) &= \mathcal{L}_{\text{adv}} + \mathcal{L}_{\text{ac}} + \mathcal{L}_{\text{exp}}\\
    \mathcal{L}_{\text{adv}} = &\ln \pi_\psi(a_\tau | s^o_\tau, s^p_\tau, h_\tau)\ \text{sg} (\mathcal{A}_\tau ) \label{eqn:pol-adv} \\
    \mathcal{L}_{\text{ac}} = &\alpha \mathbf{H} \left[ \pi_\psi(a_\tau | s^o_\tau, s^p_\tau, h_\tau) \right] \label{eqn:pol-ent} \\
    \mathcal{L}_{\text{exp}} = &-\beta m_\tau \ \odot \parallel a_\tau - a^{*}_\tau \parallel_2 \label{eqn:pol-ref}
\end{align}
\end{subequations}
where $a_\tau \sim \pi_\psi(a_\tau | s^o_\tau, s^p_\tau, h_\tau)$ is the imagined action(s) and $a^{*}_\tau$ is the actual action(s) taken. $m_\tau$ is a mask that specifies whether the action $a^{*}$ is from the expert or not. The (stability) coefficient values are $\alpha = 0.0003$ and $\beta = 0.5$. 
In addition to advantage (value) maximization -- of the term $\mathcal{L}_{\text{adv}}$ in Equation~\ref{eqn:pol-adv} -- our approach draws inspiration from the fusion of model-free RL, BC~\cite{Fujimoto2021TD3BC}, and the actor refreshing mechanism used in active predictive coding~\cite{Ororbia2023ActivePredictiveCoding}. Specifically, the term $\mathcal{L}_{\text{exp}}$ (Equation~\ref{eqn:pol-ref}) provides the agent with a small reward when its actions follow the expert demonstrations. Finally, the term $\mathcal{L}_{\text{ac}}$ (Equation~\ref{eqn:pol-ent}) 
denotes the rewards supplied to the agent for entering states where its policy incurs high entropy. This quantity is often referred to as ``curiosity''~\cite{Friston2017Curiosity} 
and has been widely used to enhance exploration. 

\newpage


\noindent 
\textbf{The Effect of Preference Regret.} 

\begin{lemma}
Minimizing preference regret as an internal reward in the advantage computation guides the agent toward preferred trajectories while maintaining the ability to maximize the final reward when preferences are sub-optimal.
\end{lemma}

\begin{proof}

We divide the proof into two cases. 

\noindent 
\textbf{Case 1: High-quality Preferences.} Assume that the reward model predicts high reward for preference states. When the agent’s imagined trajectory deviates from the predicted preference trajectory, $r^p_\tau > r^o_\tau$, it yields negative preference regret $-(r^p_\tau - r^o_\tau) $. The internal reward $\mathbf{R}_\tau$ is then penalized for deviation from the preference. Conversely, when the imagined trajectory aligns with the preference, $-(r^p_\tau - r^o_\tau) \rightarrow 0$, maintaining or increasing the advantage value, it reinforces actions that follow the preferred trajectory. 

\noindent 
\textbf{Case 2: Sub-optimal Preferences.} Assume that the preference trajectory yields lower reward than the actual policy, e.g., due to suboptimal expert demonstrations. In this case, $r^p_\tau < r^o_\tau$ results in a positive preference regret ($-(r^p_\tau - r^o_\tau) > 0 $). This contributes to the increased advantage value and thus reinforces the agent to pursue higher environmental reward ($r^o_\tau$) trajectories rather than blindly imitating the expert.

Consequently, this mechanism enables the agent to adaptively balance self / expert-imitation and real-time optimization: avoiding naive behavior cloning -- which suffers from cascading errors due to distributional shift -- while maximizing the final rewards in online learning.
\end{proof}

\begin{table}
    \centering
    \begin{tabular}{|p{4cm}|cccc|}
    \hline
        \textbf{Task/Agent} & \textbf{MYOE} & \textbf{Drm-BC} & \textbf{PPO-BC} & \textbf{LAIfO} \\ \hline
        K Kettle & \cellcolor{LightCyan}$0.97 \pm 0.06$ & $0.90 \pm 0.11$ & $0.01 \pm 0.02$ & $0.27 \pm 0.42$ \\ 
        K Light Switch & \cellcolor{LightCyan}$1.00 \pm 0.01$ & $0.51 \pm 0.49$ & $0.00 \pm 0.00$ & $0.00 \pm 0.00$ \\ 
        K Microwave & \cellcolor{LightCyan}$0.75 \pm 0.43$ & $0.67 \pm 0.34$ & $0.00 \pm 0.00$ & $0.00 \pm 0.00$ \\ 
        K Slide Cabinet & \cellcolor{LightCyan}$0.99 \pm 0.03$ & $0.74 \pm 0.43$ & $0.28 \pm 0.42$ & $0.57 \pm 0.44$ \\ 
        R Lift & \cellcolor{LightCyan}$1.00 \pm 0.01$ & $0.62 \pm 0.26$ & $0.00 \pm 0.00$ & $0.00 \pm 0.01$ \\ 
        R Door & \cellcolor{LightCyan}$0.99 \pm 0.02$ & $0.98 \pm 0.03$ & $0.00 \pm 0.00$ & $0.00 \pm 0.00$ \\ 
        M Button Press & \cellcolor{LightCyan}$1.00 \pm 0.00$ & $0.25 \pm 0.43$ & $0.04 \pm 0.05$ & $0.00 \pm 0.00$ \\ 
        M Button Press Wall & \cellcolor{LightCyan}$1.00 \pm 0.00$ & $0.00 \pm 0.00$ & $0.00 \pm 0.00$ & $0.00 \pm 0.00$ \\ 
        M Coffee Button & \cellcolor{LightCyan}$1.00 \pm 0.01$ & $0.99 \pm 0.03$ & $0.07 \pm 0.12$ & $0.54 \pm 0.36$ \\
        M Coffee Pull & \cellcolor{LightCyan}$0.77 \pm 0.13$ & $0.00 \pm 0.00$ & $0.00 \pm 0.00$ & $0.00 \pm 0.00$ \\ 
        M Coffee Push & \cellcolor{LightCyan}$0.85 \pm 0.10$ & $0.00 \pm 0.00$ & $0.00 \pm 0.00$ & $0.01 \pm 0.02$ \\ 
        M Door Close & \cellcolor{LightCyan}$1.00 \pm 0.00$ & \cellcolor{LightCyan}$1.00 \pm 0.00$ & $0.25 \pm 0.43$ & $0.25 \pm 0.43$ \\ 
        M Door Lock & \cellcolor{LightCyan}$1.00 \pm 0.02$ & $0.98 \pm 0.04$ & $0.11 \pm 0.13$ & $0.06 \pm 0.11$ \\
        M Door Unlock & \cellcolor{LightCyan}$1.00 \pm 0.00$ & \cellcolor{LightCyan}$1.00 \pm 0.00$ & $0.00 \pm 0.00$ & $0.00 \pm 0.00$ \\ 
        M Drawer Close & \cellcolor{LightCyan}$1.00 \pm 0.00$ & \cellcolor{LightCyan}$1.00 \pm 0.00$ & $0.95 \pm 0.09$ & \cellcolor{LightCyan}$1.00 \pm 0.00$ \\ 
        M Drawer Open & \cellcolor{LightCyan}$1.00 \pm 0.00$ & \cellcolor{LightCyan}$1.00 \pm 0.00$ & $0.01 \pm 0.03$ & $0.02 \pm 0.04$ \\ 
        M Faucet Close & \cellcolor{LightCyan}$1.00 \pm 0.02$ & \cellcolor{LightCyan}$1.00 \pm 0.00$ & $0.00 \pm 0.00$ & $0.00 \pm 0.00$ \\ 
        M Handle Press & \cellcolor{LightCyan}$1.00 \pm 0.00$ & $0.99 \pm 0.02$ & $0.30 \pm 0.23$ & $0.24 \pm 0.19$ \\ 
        M Handle Pull Side & \cellcolor{LightCyan}$1.00 \pm 0.00$ & $0.00 \pm 0.00$ & $0.00 \pm 0.00$ & $0.00 \pm 0.00$ \\ 
        M Handle Pull & \cellcolor{LightCyan}$1.00 \pm 0.00$ & $0.49 \pm 0.49$ & $0.02 \pm 0.04$ & $0.04 \pm 0.06$ \\ 
        M Peg Insert Side & \cellcolor{LightCyan}$0.08 \pm 0.14$ & $0.00 \pm 0.00$ & $0.00 \pm 0.00$ & $0.00 \pm 0.00$ \\ 
        M Reach & \cellcolor{LightCyan}$0.36 \pm 0.17$ & $0.02 \pm 0.04$ & $0.00 \pm 0.00$ & $0.00 \pm 0.00$ \\ 
        M Plate Slide & \cellcolor{LightCyan}$0.97 \pm 0.05$ & $0.00 \pm 0.00$ & $0.02 \pm 0.04$ & $0.18 \pm 0.33$ \\ 
        M Plate Slide Back Side & \cellcolor{LightCyan}$0.75 \pm 0.43$ & $0.00 \pm 0.00$ & $0.00 \pm 0.00$ & $0.00 \pm 0.00$ \\ 
        M Plate Slide Side & \cellcolor{LightCyan}$1.00 \pm 0.01$ & $0.99 \pm 0.02$ & $0.12 \pm 0.21$ & $0.00 \pm 0.00$ \\
        M Plate Slide Back & \cellcolor{LightCyan}$1.00 \pm 0.00$ & $0.46 \pm 0.46$ & $0.00 \pm 0.00$ & $0.00 \pm 0.00$ \\ 
        M Peg Unplug Side & \cellcolor{LightCyan}$1.00 \pm 0.01$ & $0.00 \pm 0.00$ & $0.00 \pm 0.01$ & $0.02 \pm 0.04$ \\ 
        M Soccer & \cellcolor{LightCyan}$0.30 \pm 0.11$ & $0.00 \pm 0.00$ & $0.00 \pm 0.01$ & $0.00 \pm 0.00$ \\ 
        M Reach Wall & \cellcolor{LightCyan}$0.94 \pm 0.10$ & $0.00 \pm 0.00$ & $0.00 \pm 0.00$ & $0.00 \pm 0.00$ \\ 
        M Sweep Into & \cellcolor{LightCyan}$0.47 \pm 0.17$ & $0.00 \pm 0.01$ & $0.00 \pm 0.00$ & $0.00 \pm 0.01$ \\ 
        M Window Open & \cellcolor{LightCyan}$1.00 \pm 0.00$ & \cellcolor{LightCyan}$1.00 \pm 0.00$ & $0.01 \pm 0.02$ & $0.10 \pm 0.18$ \\ 
        M Window Close & \cellcolor{LightCyan}$1.00 \pm 0.00$ & \cellcolor{LightCyan}$1.00 \pm 0.00$ & $0.00 \pm 0.00$ & $0.00 \pm 0.00$ \\ \hline
    \end{tabular}
    \caption{We report the evaluation success rate of the last $100$ episodes of MYOE (ours) as compared to other model/RL baselines. Cyan cells represent the best performing agents within the corresponding task. Note that ``Drm'' is short for ``DreamerV3.''}
    \label{tab:all-stats1}
    \vspace{-8mm}
\end{table}

\begin{table}
    \centering
    \begin{tabular}{|p{4cm}|llll|}
    \hline
        \textbf{Task/Agent} & \textbf{MYOE} & \textbf{MBC} & \textbf{MBC-VAE} & \textbf{MBC-RNN} \\ \hline
        K Kettle & \cellcolor{LightCyan}$0.97 \pm 0.06$ & $0.65 \pm 0.46$ & $0.66 \pm 0.47$ & $0.67 \pm 0.47$ \\ 
        K Light Switch & \cellcolor{LightCyan}$1.00 \pm 0.01$ & $0.23 \pm 0.22$ & $0.12 \pm 0.20$ & $0.24 \pm 0.17$ \\ 
        K Microwave & \cellcolor{LightCyan}$0.75 \pm 0.43$ & $0.50 \pm 0.50$ & $0.49 \pm 0.49$ & $0.50 \pm 0.50$ \\ 
        K Slide Cabinet & \cellcolor{LightCyan}$0.99 \pm 0.03$ & $0.08 \pm 0.12$ & $0.16 \pm 0.17$ & $0.08 \pm 0.15$ \\ 
        R Lift & \cellcolor{LightCyan}$1.00 \pm 0.01$ & $0.38 \pm 0.29$ & $0.70 \pm 0.23$ & $0.46 \pm 0.25$ \\ 
        R Door & \cellcolor{LightCyan}$0.99 \pm 0.02$ & $0.96 \pm 0.06$ & $0.81 \pm 0.26$ & $0.97 \pm 0.04$ \\ 
        M Button Press & \cellcolor{LightCyan}$1.00 \pm 0.00$ & $0.00 \pm 0.00$ & $0.08 \pm 0.10$ & $0.29 \pm 0.24$ \\ 
        M Button Press Wall & \cellcolor{LightCyan}$1.00 \pm 0.00$ & $0.41 \pm 0.41$ & $0.49 \pm 0.42$ & $0.00 \pm 0.00$ \\ 
        M Coffee Button & \cellcolor{LightCyan}$1.00 \pm 0.01$ & $0.62 \pm 0.31$ & $0.82 \pm 0.13$ & $0.15 \pm 0.16$ \\ 
        M Coffee Pull & \cellcolor{LightCyan}$0.77 \pm 0.13$ & $0.10 \pm 0.17$ & $0.05 \pm 0.06$ & $0.07 \pm 0.09$ \\ 
        M Coffee Push & \cellcolor{LightCyan}$0.85 \pm 0.10$ & $0.14 \pm 0.11$ & $0.24 \pm 0.13$ & $0.10 \pm 0.13$ \\ 
        M Door Close & \cellcolor{LightCyan}$1.00 \pm 0.00$ & $0.04 \pm 0.08$ & $0.36 \pm 0.35$ & $0.27 \pm 0.42$ \\ 
        M Door Lock & \cellcolor{LightCyan}$1.00 \pm 0.02$ & $0.21 \pm 0.14$ & $0.57 \pm 0.09$ & $0.36 \pm 0.24$ \\ 
        M Door Unlock & \cellcolor{LightCyan}$1.00 \pm 0.00$ & $0.33 \pm 0.24$ & $0.36 \pm 0.15$ & $0.34 \pm 0.09$ \\ 
        M Drawer Close & \cellcolor{LightCyan}$1.00 \pm 0.00$ & \cellcolor{LightCyan}$1.00 \pm 0.00$ & $0.99 \pm 0.03$ & $0.99 \pm 0.02$ \\ 
        M Drawer Open & \cellcolor{LightCyan}$1.00 \pm 0.00$ & $0.31 \pm 0.34$ & $0.80 \pm 0.17$ & $0.02 \pm 0.04$ \\ 
        M Faucet Close & \cellcolor{LightCyan}$1.00 \pm 0.02$ & $0.71 \pm 0.29$ & $0.77 \pm 0.12$ & $0.72 \pm 0.24$ \\ 
        M Handle Press & \cellcolor{LightCyan}$1.00 \pm 0.00$ & $0.65 \pm 0.20$ & $0.72 \pm 0.26$ & $0.64 \pm 0.15$ \\ 
        M Handle Pull Side & \cellcolor{LightCyan}$1.00 \pm 0.00$ & $0.03 \pm 0.06$ & $0.00 \pm 0.00$ & $0.02 \pm 0.03$ \\ 
        M Handle Pull & \cellcolor{LightCyan}$1.00 \pm 0.00$ & $0.28 \pm 0.22$ & $0.27 \pm 0.18$ & $0.50 \pm 0.26$ \\ 
        M Peg Insert Side & \cellcolor{LightCyan}$0.08 \pm 0.14$ & $0.04 \pm 0.07$ & $0.01 \pm 0.02$ & $0.06 \pm 0.06$ \\ 
        M Reach & \cellcolor{LightCyan}$0.36 \pm 0.17$ & $0.19 \pm 0.13$ & $0.20 \pm 0.13$ & $0.21 \pm 0.23$ \\ 
        M Plate Slide & \cellcolor{LightCyan}$0.97 \pm 0.05$ & $0.00 \pm 0.00$ & $0.08 \pm 0.14$ & $0.16 \pm 0.19$ \\ 
        M Plate Slide Back Side & \cellcolor{LightCyan}$0.75 \pm 0.43$ & $0.24 \pm 0.41$ & $0.66 \pm 0.32$ & $0.49 \pm 0.40$ \\ 
        M Plate Slide Side & \cellcolor{LightCyan}$1.00 \pm 0.01$ & $0.44 \pm 0.45$ & $0.71 \pm 0.22$ & $0.24 \pm 0.41$ \\ 
        M Plate Slide Back & \cellcolor{LightCyan}$1.00 \pm 0.00$ & $0.77 \pm 0.15$ & $0.86 \pm 0.11$ & $0.42 \pm 0.36$ \\ 
        M Peg Unplug Side & \cellcolor{LightCyan}$1.00 \pm 0.01$ & $0.20 \pm 0.14$ & $0.23 \pm 0.23$ & $0.17 \pm 0.18$ \\ 
        M Soccer & \cellcolor{LightCyan}$0.30 \pm 0.11$ & $0.03 \pm 0.04$ & $0.20 \pm 0.11$ & $0.23 \pm 0.15$ \\ 
        M Reach Wall & \cellcolor{LightCyan}$0.94 \pm 0.10$ & $0.18 \pm 0.11$ & $0.09 \pm 0.08$ & $0.15 \pm 0.14$ \\ 
        M Sweep Into & \cellcolor{LightCyan}$0.47 \pm 0.17$ & $0.08 \pm 0.12$ & $0.17 \pm 0.08$ & $0.06 \pm 0.11$ \\ 
        M Window Open & \cellcolor{LightCyan}$1.00 \pm 0.00$ & $0.30 \pm 0.27$ & $0.45 \pm 0.19$ & $0.34 \pm 0.15$ \\ 
        M Window Close & \cellcolor{LightCyan}$1.00 \pm 0.00$ & $0.81 \pm 0.18$ & $0.93 \pm 0.09$ & $0.88 \pm 0.13$ \\ \hline
    \end{tabular}
    \caption{We report the evaluation success rate of the last $100$ episodes of MYOE (ours) in comparison to other baselines. Cyan cells represent the best performance agents within the corresponding task.}
    \label{tab:all-stats2}
    \vspace{-8mm}
\end{table}

\section{Experimental Results}
\label{sec:exp}

\noindent 
\textbf{Baselines.} We seek to demonstrate that our agent outperforms state-of-the-art RLfD baselines in a limited expert data scenario. Specifically, we compare our agent (MYOE) to relevant model-based and model-free RL with IL-augmentation and adversarial IL schemes, namely: 
\textbf{1)} DreamerV3~\cite{Hafner2025DreamerV3} with BC (Drm-BC); 
\textbf{2)} PPO~\cite{Schulman2017PPO, Schulman2015GeneralGAELambda} with BC (PPO-BC); and 
\textbf{3)} PPO with adversarial IL~\cite{Giammarino2024LAIfO, Ho2016GAIL} (LAIfO). Additionally, we also compare to self-supervised IL systems, i.e., BC agents that train on their successful data in addition to the expert data, to demonstrate the robustness of ``preference regret'' optimization over the cascading error. These particular baseline systems includes multimodal BC (MBC), MBC with recurrent neural network (MBC-RNN), and MBC with variational autoencoder~\cite{Kingma2014VAE} (MBC-VAE). All of the agents' total number of parameters are constrained to $\approx 12$ million and each agent trains with the environment for $1$ million steps. For each agent and environment, we conduct $4$ trials and measure performance by calculating the mean and standard deviation of the success rate (SR) from the last $100$ non-learning episodes and their corresponding cumulative rewards.

\noindent 
\textbf{Benchmarks.} We evaluate our proposed agent and relevant baselines on three sets of simulated environments consisting of $32$ tasks in total: 
$4$ Franka Kitchen~\cite{Gupta2020RelayPolicyLearning} (K) tasks, 
$26$ Meta-World~\cite{Yu2020MetaWorld} (M) tasks, and 
$2$ Robosuite~\cite{Zhu2020Robosuite} (R) tasks. 
All of the environments have a sparse reward setting and use visual frames and states as observation, with the Franka Kitchen providing language as an additional modality. For each environment, we collect only $5$ episodes of expert demonstration data to use in the sparse reward setting and all agents / baselines benefit from this data pool (each episode consists of $\approx 250$, $200$, and $1000$ steps in Franka Kitchen, Meta-World, and Robosuite, respectively). All environments have been modified to encourage cascading errors, (i.e., Franka Kitchen adds noise to the robot actions whereas Robosuite and Meta-World change their goal in each episode).

\noindent 
\textbf{Simulation Results.} Overall, our agent converges faster for most tasks in comparison to other baselines; see Figure~\ref{fig:cumrew}. This result demonstrates that our proposed method is capable of utilizing potentially suboptimal and limited demonstrations while still being able to optimize for the task at hand. This is due to the fact that our agent continuously optimizes its preference, which guides its policy toward expert policies when they are optimal and toward overall reward signal values regardless of the expert given that some expert policies might be suboptimal. Baseline RLfD algorithms can also converge in many tasks because of the direct expert action signal; however, some expert policies are not optimal (such as M Window Close), affecting the training performance of the baseline RLfD algorithms; see Figure~\ref{fig:cumrew}, Table~\ref{tab:all-stats1}, and~\ref{tab:all-stats2}. Additionally, sparse reward settings reduce the effectiveness of traditional RLfD algorithms, hindering their learning ability (such as in LAIfO and PPO-BC). 
Finally, we observe that baseline RLfD algorithms and self-supervised IL systems does not optimize at all in some tasks (such as in M Button Press Wall), suggesting that they suffer from cascading errors, resulting in the acquisition of policies that have diverged.


\begin{figure}[!t]
    \centering
    \includegraphics[width=0.8\linewidth]{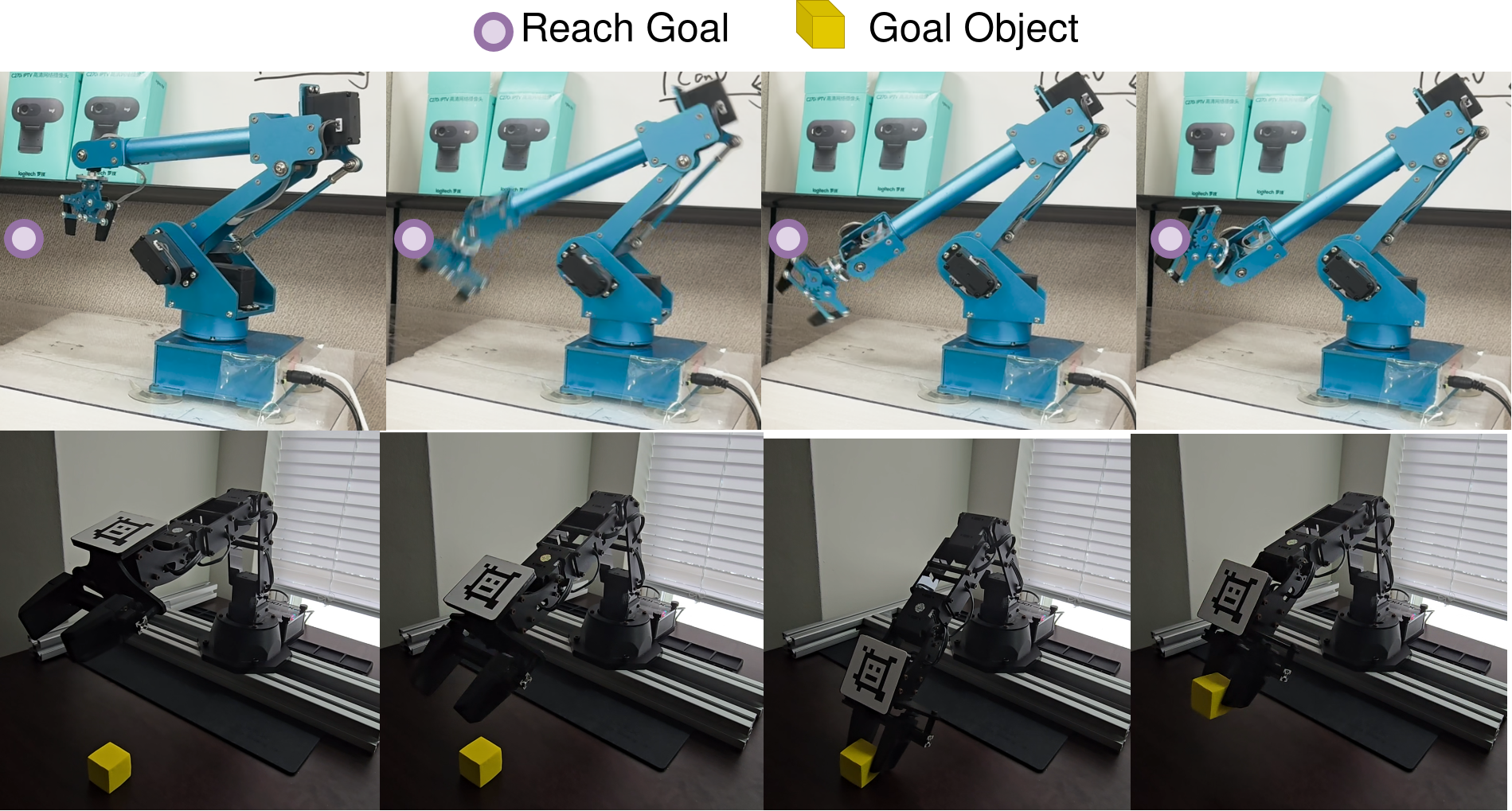}
    \caption{Our proposed MYOE solving the ``reach'' task when integrated into the ``7bot'' robot (top) and our agent model solving the ``block picking'' task when integrated into the PX100 robot (bottom).}
    \label{fig:robots}
    \vspace{-6mm}
\end{figure}


\noindent 
\textbf{Real-World Neurorobot Online Learning.} To further investigate our method's robustness compared to state-of-the-art RLfD baselines, we conducted two real-world online RLfD experiments using the ``7bot''~\cite{7bot20157bot} (6-DOF) and the PX100 robot~\cite{Trossen2018PX100} (4-DOF). The first task requires the ``7bot'' to reach a position given a target joint configuration and the second task requires the PX100 robot to pickup a block. Due to the computational and logistical constraints inherent to real-world in-robot online learning -- including the non-parallelizable nature of physical systems -- we focus our comparison on $4$ representative agents over $4$ trials: our MYOE model, DRM-BC, PPO-BC, and MBC-RNN (for ``7bot'') / LAIfO (for PX100 robot). To simulate realistic learning scenarios with limited expert data, we constrain each agent to train for only $125$ episodes while providing $5$ episodes of demonstrations ($10$ steps per episode). The expert data is recorded to encourage cascading errors (such as shaking the robot arm before picking up the block). We evaluate performance by recording the mean and standard deviation of the task success rate (and task cumulative reward for the block picking task) over the last $20$ episodes. The results show that, \textit{across all tasks, MYOE outperforms all other RLfD and self-supervised IL baselines} (see Table~\ref{tab:robot-results},~\ref{tab:robot-results2}, and Figure~\ref{fig:robots}). This validates that optimizing preference regret provides notable advantages for neurorobotic task optimization, especially in resource-constrained RLfD scenarios where cascading errors is mostly present.

\begin{table}[!t]
    \centering
    \vspace{1mm}
    \begin{tabular}{cccc}
    \textbf{MYOE} & \textbf{Drm-BC} & \textbf{PPO-BC}  & \textbf{MBC-RNN} \\
      \hline
    \cellcolor{LightCyan} 0.34 $\pm$ 0.11  &  0.17 $\pm$ 0.01  & 0.29 $\pm$ 0.11 &  0.00 $\pm$ 0.00 
    \end{tabular}
    \caption{``Reach'' task success rate on the ``7bot''.}
    \label{tab:robot-results}
    \vspace{-10mm}
\end{table}

\begin{table}[!t]
    \centering
    \vspace{1mm}
    \begin{tabular}{cccc}
     \textbf{MYOE} & \textbf{Drm-BC} & \textbf{PPO-BC}  & \textbf{LAIfO} \\
      \hline
    \cellcolor{LightCyan} 0.89 $\pm$ 0.05  &  0.81 $\pm$ 0.05 & 0.78 $\pm$ 0.06 &  0.79 $\pm$ 0.06 \\
    \cellcolor{LightCyan} 80.78 $\pm$ 4.48  &  72.34 $\pm$ 4.93  & 66.78 $\pm$ 6.30 &  67.96 $\pm$ 6.19
    \end{tabular}
    \caption{``Block picking'' task success rate (top) and episodic cumulative reward (bottom) on the PX100 robot.}
    \label{tab:robot-results2}
    \vspace{-6mm}
\end{table}

\section{Conclusions}
\label{sec:conclusions}

In this work, we constructed an RLfD agent framework that optimizes (neuro)robotic task solvers using few expert demonstration samples. We proposed a perception learning paradigm, QMoP-SSM, that makes use of goal learning and preference trajectory estimation. Furthermore, we formulated an agent behavioral learning framework that leverages preference information from the perception model. Specifically, by introducing a concept called ``preference regret'', our agent's policy is capable of learning tasks quickly and accurately as compared to other state-of-the-art model-based and model-free RL (with IL-augmentation and self-supervised IL) systems on both simulation and real-world robot platforms.

\noindent 
\textbf{Limitations and Future Work.} Although MYOE outperforms other state-of-the-art baselines, for more complex tasks such as ``M Soccer,'' our approach's performance is not likely to surpass that of experts. We assume that these scenarios are too difficult (due to small object pixel areas) such that our preference learning and imagination scheme would struggle to effectively utilize the image observation and properly extract useful temporal dependencies within the data. 
Future work should consider improving MYOE's preference trajectory estimation, utilizing different autoregressive models such as deep Kalman filters~\cite{kalman1960filter, Krishnan2015DeepKalmanFilters}. 
Finally, one could considering enhancing our architecture's ability to encode information in its posterior from subtle pixel changes in order to better capture small objects within a (neuro)robot's workspace. 

\bibliographystyle{IEEEtran}
\bibliography{IEEEabrv, main}

\end{document}